%% file: main.tex
\documentclass[11pt]{cambrian}
\usepackage[authoryear,round]{natbib}
\usepackage{float,flafter,tikz}
\usetikzlibrary{arrows.meta}
\definecolor{accent}{HTML}{1A5276}
\definecolor{warn}{HTML}{7B241C}
\definecolor{scholarblue}{rgb}{0.21,0.49,0.74}
\hypersetup{colorlinks=true,citecolor=scholarblue,linkcolor=scholarblue,urlcolor=scholarblue}

\newtheorem{theorem}{Theorem}

\theoremstyle{definition}
\newtheorem{definition}[theorem]{Definition}

\newcommand{\Ex}{\mathbb{E}}

\title{Least Squares for Time Series Forecasting}

\author{Weiu-qiou Ciang, Yuzhou Hong, and Sherry Chen}

\begin{abstract}
A time-series forecast is scored on a future value of the series.
A representation loss that regresses the next latent, as in LeNEPA, is a different least-squares problem on the same bottleneck.
We write both programs down.
The forecast program minimizes the error of a decoded latent on the coordinate that will be reported.
For a scalar target and a linear decoder, every latent rank of at least one matches ordinary least squares, and an isotropy constraint is only a rescaling: after the decoder is refit, the forecast does not move.
The other program fits the whole next vector at a fixed rank, then freezes the encoder and attaches a head.
On a four-dimensional series whose last three coordinates are the same autoregression, that rank-1 fit puts mass $0.9998$ on the repeated coordinate and forecasts the remaining signal at the marginal variance $2.794$.
The forecast program puts mass $1$ on the signal and matches the innovation variance $0.992$.
Rank $2$ gives the vector fit a second direction, and the two programs agree.
Iterating the fitted one-step coefficient $0.803$ raises the open-loop error from $0.992$ at one step to $2.700$ at eight steps.
\end{abstract}

\begin{document}
\maketitle

\section{Introduction}

Predicting the next latent of a time series is a representation objective.
\citet{chemeris2026lenepa} train a causal transformer so that a map sends the current latent token to the next one, and they replace the stop-gradient of \citet{xu2025nextembeddingpredictionmakesstrong} with the isotropy penalty of \citet{balestriero20258544}.
The number that training sees is an error in the latent.
A forecasting deployment is scored on a future value of the series.
Those two errors are different functions of the encoder.

This paper writes the design choice as a pair of least-squares programs.
Let $y_t\in\mathbb{R}^d$ be the observation and let $s_{t+1}\in\mathbb{R}$ be the coordinate that the forecast must hit.
A latent $z_t=Wy_t$ with $W\in\mathbb{R}^{k\times d}$ is a linear bottleneck.
The forecast program chooses $W$ and a head $b\in\mathbb{R}^k$ to minimize the error in $s$.
The next-latent program chooses $W$ and a map into $\mathbb{R}^d$ to minimize the error in the whole next vector $y_{t+1}$, then fits $b$ with $W$ held fixed.
Isotropy, $\Ex[zz^\top]=I$, is a constraint on the marginal of $z$.
It does not mention $s_{t+1}$.

Figure~\ref{fig:programs} is the whole design.
The top row is what a forecast report will average.
The bottom row is what a LeNEPA-style step averages, before any head is attached.
Section~\ref{sec:opt} states the two programs and three facts about their solutions.
The solutions are normal equations, not a search over architectures.
Section~\ref{sec:exp} evaluates them on one synthetic series whose redundancy is known, so a wrong direction is visible in the weights.
The recommendation is narrow: put the training loss on the reported coordinate, treat isotropy as a gauge that is absorbed by refitting the head, and publish the open-loop error of the fitted one-step map.

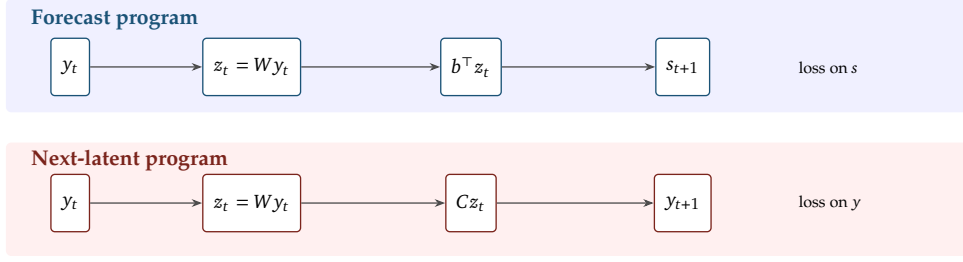
\begin{figure}[t]
\centering
\begin{tikzpicture}[
  font=\scriptsize,
  box/.style={draw=accent, rounded corners=1.5pt, fill=white, line width=0.5pt,
              minimum height=4.2ex, inner xsep=4pt, align=center},
  arr/.style={-{Stealth[length=1.4mm]}, line width=0.5pt, draw=black!70}
]
  \fill[blue!6, rounded corners=2pt] (-0.15,0.35) rectangle (12.6,1.85);
  \fill[red!6, rounded corners=2pt] (-0.15,-1.55) rectangle (12.6,-0.05);
  \node[anchor=west, font=\scriptsize\bfseries, text=accent] at (0.05,1.58) {Forecast program};
  \node[box] (y) at (0.7,0.95) {$y_t$};
  \node[box] (z) at (3.1,0.95) {$z_t=Wy_t$};
  \node[box] (b) at (6.0,0.95) {$b^\top z_t$};
  \node[box] (s) at (8.8,0.95) {$s_{t+1}$};
  \node[anchor=west, font=\tiny] at (10.2,0.95) {loss on $s$};
  \draw[arr] (y)--(z); \draw[arr] (z)--(b); \draw[arr] (b)--(s);
  \node[anchor=west, font=\scriptsize\bfseries, text=warn] at (0.05,-0.28) {Next-latent program};
  \node[box, draw=warn] (y2) at (0.7,-0.85) {$y_t$};
  \node[box, draw=warn] (z2) at (3.1,-0.85) {$z_t=Wy_t$};
  \node[box, draw=warn] (c) at (6.0,-0.85) {$C z_t$};
  \node[box, draw=warn] (yn) at (8.8,-0.85) {$y_{t+1}$};
  \node[anchor=west, font=\tiny] at (10.2,-0.85) {loss on $y$};
  \draw[arr] (y2)--(z2); \draw[arr] (z2)--(c); \draw[arr] (c)--(yn);
\end{tikzpicture}
\caption{Two least-squares programs on the same bottleneck $W$.
The forecast program is scored on $s_{t+1}$.
The next-latent program is scored on every coordinate of $y_{t+1}$, and a head for $s$ is fit only after $W$ is frozen.}
\label{fig:programs}
\end{figure}

\section{Two programs}
\label{sec:opt}

\subsection{Setup}

Observations $y_t\in\mathbb{R}^d$ are square-integrable.
The reported target $s_{t+1}$ is a fixed linear coordinate of the future.
In the experiment it is the first coordinate of $y_{t+1}$, and the other coordinates are a repeated nuisance series.
A rank-$k$ linear latent is $z_t=Wy_t$ with $W\in\mathbb{R}^{k\times d}$.
Write $\langle u,v\rangle$ for the empirical inner product on the sample, and center every vector that enters a covariance.
The population statements use $\Ex$; the experiment replaces $\Ex$ by the average over $T=20000$ times.

\begin{definition}[Forecast program]
\label{def:f}
\[
\begin{aligned}
\min_{W,b}\quad & \Ex\bigl(s_{t+1}-b^\top W y_t\bigr)^2 \\
\mathrm{s.t.}\quad & \Ex[z_t z_t^\top]=I_k.
\end{aligned}
\]
The unconstrained problem drops the constraint.
The product $\beta^\top=b^\top W\in\mathbb{R}^{1\times d}$ is the forecast coefficient.
\end{definition}

\begin{definition}[Next-latent program]
\label{def:n}
\[
\min_{W,C}\quad \Ex\|y_{t+1}-C W y_t\|^2,
\qquad C\in\mathbb{R}^{d\times k},
\]
which forces $\mathrm{rank}(CW)\le k$.
After a minimizer $W_\star$ is obtained, the head is the scalar regression
\[
\min_b\ \Ex\bigl(s_{t+1}-b^\top W_\star y_t\bigr)^2.
\]
\end{definition}

Definition~\ref{def:n} is the linear version of a next-embedding step followed by a probe.
\citet{chemeris2026lenepa} use a nonlinear encoder and an isotropy penalty in place of the hard constraint in Definition~\ref{def:f}.
The linear case is the one in which the two programs have normal equations and the experiment can check them.

\subsection{The forecast program is ordinary least squares}

\begin{theorem}[Scalar targets ignore extra rank]
\label{thm:ols}
The unconstrained forecast program is ordinary least squares in $\beta\in\mathbb{R}^d$.
Its solution satisfies
\[
\Ex[y_t y_t^\top]\,\beta=\Ex[y_t s_{t+1}],
\]
whenever the Gram matrix is invertible.
For every $k\ge 1$ there is a factorization $\beta=W^\top b$ attaining the same value.
Thus the minimal forecast risk does not decrease as $k$ grows past $1$.
\end{theorem}

\begin{proof}
Expanding the square and differentiating in $\beta$ gives the normal equation.
The Hessian is the Gram matrix, so the critical point is the unique minimizer when the Gram matrix is positive definite.
If $k\ge 1$ and $\beta\neq 0$, the choice $W$ with first row $\beta^\top/\|\beta\|$ and $b=\|\beta\|\,e_1$ recovers $\beta$.
If $\beta=0$, take $b=0$.
Extra rows of $W$ can be anything the head sets to zero.
\end{proof}

The practical reading is a design rule.
A linear forecast of one number does not need a wide latent.
Width starts to matter when the decoder cannot represent the optimal $\beta$, or when the target itself is a vector whose predictable subspace has rank greater than one.
That second case is exactly the next-latent program.

\subsection{Isotropy is a gauge}

\begin{theorem}[Whitening does not move $\beta$]
\label{thm:gauge}
Suppose $(W,b)$ is feasible for the unconstrained forecast program and $Q\in\mathbb{R}^{k\times k}$ is invertible.
Then $W'=QW$ and $b'=Q^{-\top}b$ define the same coefficient $b'^\top W'=b^\top W$ and the same forecast risk.
If $k=1$ and $\Ex[z_t^2]=\sigma_z^2>0$, the choice $Q=\sigma_z^{-1}$ makes the new latent have unit second moment, and the head is multiplied by $\sigma_z$.
The constrained and unconstrained minimal risks therefore coincide.
\end{theorem}

\begin{proof}
The product identity is $b'^\top W'=b^\top Q^{-1} Q W=b^\top W$.
For $k=1$, $z'_t=z_t/\sigma_z$ has unit second moment, and $b'=b\sigma_z$ restores the product.
Any constrained pair is feasible for the unconstrained problem, so the constrained minimum cannot be smaller.
The rescaling shows it is not larger.
\end{proof}

For $k>1$ the same argument applies with a matrix square root of the latent covariance, on the range of that covariance.
Rotations that preserve $I_k$ are the residual gauge: they change $W$ and $b$ and leave $b^\top W$ fixed.
\citet{balestriero20258544} introduce SIGReg so that a joint-embedding representation does not collapse.
\citet{bardes20224906} use a covariance penalty for the same reason.
In a forecast program whose head is refit after the penalty, those constraints fix coordinates of $z$.
They are not an extra source of forecast accuracy.
Section~\ref{sec:exp} records the finite-sample version: whitening the least-squares latent and refitting $b$ changes the one-step MSE by less than $10^{-15}$.

\subsection{The next-latent program is reduced-rank regression}

\begin{theorem}[A different normal equation]
\label{thm:rrr}
A minimizer of Definition~\ref{def:n} is a reduced-rank regression of $y_{t+1}$ on $y_t$.
If $G=\Ex[y_t y_t^\top]$ is invertible and $C_{ys}=\Ex[y_{t+1} y_t^\top]$, the retained $k$ directions are the leading eigenvectors of $C_{ys} G^{-1} C_{ys}^\top$.
The subsequent scalar head only sees the column space that this eigenproblem kept.
That column space depends on every coordinate of $y_{t+1}$, with weight equal to that coordinate's contribution to the squared error.
\end{theorem}

\begin{proof}
For an unrestricted matrix the normal equation is $GB^\top=C_{ys}^\top$, or $B^\top=G^{-1}C_{ys}^\top$.
Imposing $\mathrm{rank}(B)\le k$ keeps the leading eigenspace of $C_{ys} G^{-1} C_{ys}^\top$.
The finite-sample singular vectors are in the appendix.
The head is ordinary least squares of $s_{t+1}$ on $W_\star y_t$.
If $s_{t+1}$ is orthogonal to that latent, the head returns the marginal predictor.
\end{proof}

\subsection{Open loop of a one-step coefficient}

The programs above are one-step.
A horizon-$K$ forecast that iterates a scalar coefficient $a$ on an AR(1) with innovation variance $\sigma^2$ has risk
\begin{equation}
R_K=\sigma^2\frac{1-a^{2K}}{1-a^2},\qquad a^2\neq 1.
\label{eq:bill}
\end{equation}
This is the same residual sum as in a linear world-model rollout, and it is determined by the one-step fit.
It is not equal to $R_1$ at later horizons.
A forecasting paper that reports only the training loss is reporting $R_1$.

\subsection{Why a repeated coordinate wins}

Take two independent stationary AR(1) signals with the same coefficient $a$ and the same variance $v$.
Pack them as $y_t=(s_t,u_t,u_t,u_t)$.
The conditional mean of the next vector is
\[
\Ex[y_{t+1}\mid y_t]=(a s_t,\ a u_t,\ a u_t,\ a u_t).
\]
A rank-$1$ map can match only one direction of this mean.
In the squared Euclidean metric the three copies of $a u_t$ contribute three identical residuals, and $a s_t$ contributes one.
The leading direction is therefore the $u$-block.
A head that is allowed to see only that direction, and is asked for $s_{t+1}$, has nothing to regress: $s$ is independent of $u$, so the head returns the marginal predictor.
A second direction can carry $s$, which is why rank $2$ closes the gap in Table~\ref{tab:mse}.
The forecast program never faces this choice.
Its loss has only the $s$ residual, so the normal equation puts the weight on the first coordinate.

\section{A series that separates the programs}
\label{sec:exp}

Two independent AR(1) processes are driven by standard normal innovations,
\[
s_{t+1}=0.8\,s_t+\varepsilon^s_{t+1},\qquad u_{t+1}=0.8\,u_t+\varepsilon^u_{t+1},
\]
with $s_0=u_0=0$ and length $T=20000$, seed $0$.
The observation repeats the nuisance three times,
\[
y_t=(s_t,\ u_t,\ u_t,\ u_t)\in\mathbb{R}^4.
\]
The reported target is $s_{t+1}$.
Because $u$ occupies three coordinates, a squared error on $y_{t+1}$ weights $u$ three times as heavily as $s$.
A rank-$1$ fit of the whole vector therefore prefers $u$.
A fit of $s$ alone prefers the first coordinate.

Both programs are solved by the normal equations on this sample, not by an early-stopped optimizer.
The forecast coefficient is least squares of $s_{t+1}$ on $y_t$.
The next-latent coefficient is the rank-$k$ truncation described in the appendix, followed by least squares of $s_{t+1}$ on the retained latent.
Isotropy is applied only as the rescaling in Theorem~\ref{thm:gauge}.
Code and the JSON of every number below are in \texttt{experiments/run\_lfo.py}.

\begin{table}[H]
\caption{One-step MSE of $s_{t+1}$ on the repeated-$u$ series.
The innovation row uses the true coefficient $0.8$.
The marginal row predicts the sample mean of $s$.
Mass on $s$ is the squared share of the encoder weight on the first coordinate.}
\label{tab:mse}
\centering
\small
\begin{tabular}{llrr}
\toprule
Program & rank $k$ & MSE of $s$ & mass on $s$ \\
\midrule
Forecast & $\ge 1$ & 0.9921 & 1.0000 \\
Next-latent & 1 & 2.7938 & 0.0002 \\
Next-latent & 2 & 0.9921 & \textemdash \\
Innovation of $0.8\,s_t$ & \textemdash & 0.9922 & \textemdash \\
Marginal of $s$ & \textemdash & 2.7940 & \textemdash \\
\bottomrule
\end{tabular}
\end{table}

Table~\ref{tab:mse} is Theorem~\ref{thm:ols} and Theorem~\ref{thm:rrr} on one sample.
The forecast program matches the innovation variance.
Its weight on the first coordinate is $1$ up to $10^{-8}$, and the three copies of $u$ receive the rest.
The rank-$1$ next-latent program does the opposite: mass $0.9998$ on the $u$-block, and the head for $s$ cannot beat the marginal variance $2.7940$ (achieved MSE $2.7938$).
Rank $2$ gives that program a second direction.
The MSE falls to $0.9921$ and agrees with the forecast program.
The extra rank was spent putting $s$ back into the latent, which the forecast loss had selected immediately.

Whitening the forecast latent and refitting the head leaves the MSE at $0.992115558995485$ on both sides of the rescaling.
That is Theorem~\ref{thm:gauge} in floating point.
An isotropy penalty is the right tool when a representation must not collapse \citep{balestriero20258544,chemeris2026lenepa}.
It is not a second forecast loss.

\begin{figure}[H]
\centering
\includegraphics[width=\textwidth]{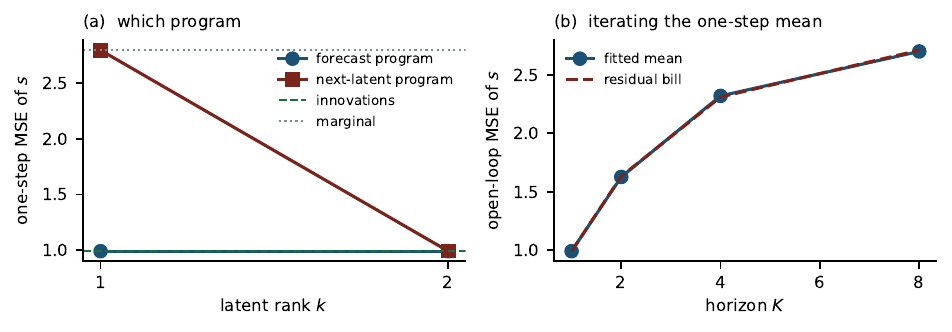}
\caption{Seed $0$, $T=20000$.
(a) One-step MSE of $s$ against latent rank.
The forecast program is already optimal at rank $1$.
The next-latent program matches it only after a second direction is available.
(b) Open-loop MSE of iterating the scalar least-squares coefficient $0.8031$ on $s$, against the residual bill \eqref{eq:bill} at that coefficient.}
\label{fig:panels}
\end{figure}

The one-step coefficient fit to the scalar series $s$ alone is $0.8031$, against the simulation value $0.8$.
Iterating it without feeding back new observations produces the open-loop errors in Table~\ref{tab:roll}.
The bill \eqref{eq:bill} uses this coefficient and the oracle innovation variance $0.9922$.
The gap at horizon $8$ is $0.011$, from the transient that starts at $s_0=0$ and from using the oracle innovation rather than the residual of $0.8031$.
The shape is the design point: a solved one-step program still has a horizon curve, and the curve is fixed by $(a,\sigma^2)$.

\begin{table}[H]
\caption{Open-loop MSE of $s$ under the scalar coefficient.
Fitted uses $a=0.8031$.
Oracle uses $a=0.8$.}
\label{tab:roll}
\centering
\small
\begin{tabular}{lrrrr}
\toprule
 & $K=1$ & $K=2$ & $K=4$ & $K=8$ \\
\midrule
Fitted iteration & 0.9921 & 1.6268 & 2.3202 & 2.7000 \\
Oracle $0.8$ & 0.9922 & 1.6269 & 2.3201 & 2.7004 \\
Bill at $a=0.8031$ & 0.9922 & 1.6320 & 2.3108 & 2.7106 \\
\bottomrule
\end{tabular}
\end{table}

\section{What to optimize}

The experiment is one linear series, not a leaderboard.
It is large enough that the normal equations are not a small-sample accident, and simple enough that the weights can be read.
Three design choices follow from the programs.

Score the coordinate that the paper will report.
If the loss averages all coordinates of $y_{t+1}$, repeated or high-variance nuisance channels own the leading singular vector.
On this series that vector is the $u$-block, and a frozen head cannot recover $s$.
Next-latent prediction remains the right training loss for a representation that must predict every channel \citep{chemeris2026lenepa,xu2025nextembeddingpredictionmakesstrong}.
It is a mismatched loss for a single reported series.

Refit the head after any marginal penalty.
Theorem~\ref{thm:gauge} says the forecast coefficient is constant on invertible reparameterizations of $z$.
SIGReg and the covariance penalty of VICReg move $z$ along that orbit when the original fit was only a scale away from isotropy \citep{balestriero20258544,bardes20224906}.
The experiment's whitened and unwhitened MSE agree to machine precision.
A penalty that is not followed by a refit of $b$ can look like a change in forecast error when it is only a change in units.

Report the horizon curve of the fitted one-step map.
Equation \eqref{eq:bill} is closed form for an AR(1).
Table~\ref{tab:roll} is that formula next to the iterated coefficient.
Linear baselines that are hard to beat on long horizons \citep{zeng20223504} are consistent with this: once the state is the window and the conditional mean is linear, a deeper next-token model does not change $\beta$.
Transformers trained on long forecast windows \citep{zhou20217436,nie20234730,wu20232186,wu20223008,liu20246625} are a different program when their training loss really is the multi-horizon error.
A one-step latent loss plus a probe is not a substitute for that program.
Time-series foundation models that are scored on those windows \citep{ansari20247815,yue20220466,dong20230861,ekambaram20239364} should be read by which program their loss implements.
Interactive video world models that distill a one-step generator and then train on their own rollouts \citep{robbyant2026lingbotworld,gao2026lingbotinfinity,zhang2026echowm} are facing the same gap at the level of pixels.
This paper does not scale that distillation.
It identifies the least-squares problem a forecaster should write down before the architecture is chosen.

\section{Related work}

\subsection{Representation losses}

\citet{oord2018cpc} score future latents with a density ratio rather than a squared residual.
\citet{xu2025nextembeddingpredictionmakesstrong} predict the next embedding of an image and stop the gradient so the target cannot collapse.
\citet{chemeris2026lenepa} carry that regression to time series with a causal transformer \citep{vaswani2017transformer}, no augmentation, and the isotropy penalty of \citet{balestriero20258544} in place of the stop-gradient.
\citet{bardes20224906} use a covariance penalty for the same collapse problem.
The number these methods publish is a property of the representation.
A forecast report averages a future observation.
Definition~\ref{def:f} is that report written as a program.
Definition~\ref{def:n} is the representation step, followed by a frozen head.

\subsection{Latent world models}

A controller that imagines needs a transition it can iterate.
\citet{ha20180122} learn a latent dynamics model and roll it inside the policy.
\citet{hafner20194551} plan in a latent model trained from pixels.
\citet{hafner20201603} train the policy on trajectories drawn from the transition, so the learned object is a kernel.
\citet{hafner20222193} keep that kernel discrete for Atari, and \citet{hafner20244104} keep the imagined trajectory as the training domain across many control suites.
\citet{wu20224176} run the same imagination loop on a physical robot.
\citet{schrittwieser20208265} plan with a learned model whose outputs are values, policies, and rewards.
\citet{hansen20224955} and \citet{hansen20246828} learn a latent dynamics model for model-predictive control and then scale it.
\citet{micheli20230588} roll a transformer world model out for sample-efficient control.
\citet{nagabandi20191652} fit deep dynamics for manipulation.
On the policy side, \citet{haarnoja20181290} and \citet{schulman20176347} assume a trajectory already exists to optimize against.
Table~\ref{tab:roll} is the linear version of the quantity these systems pay: the error of iterating a one-step map, which is larger than the one-step training loss.

\subsection{Video and interactive world models}

Pixel generators now play the role of the transition.
\citet{alonso20242399} generate the next frame with diffusion and use it as a world model, and \citet{valevski20254837} push a diffusion world model toward interactive rates.
\citet{bruce20245391} learn an action-conditioned interactive environment.
\citet{hu20237080} train a generative world model for driving.
\citet{janner20229991} plan by diffusing a whole trajectory rather than a single next state.
\citet{chen20241392} mix next-token prediction with diffusion along time.
\citet{robbyant2026lingbotworld} distill a video model into a causal interactive simulator with a minute-scale horizon.
\citet{gao2026lingbotinfinity} train that causal generator for an unbounded horizon and distill a real-time model.
\citet{zhang2026echowm} answer a continuous camera trajectory with video, environmental sound, and speech, and train the long generation on the model's own previous frames.
Each of these systems has a one-step denoiser or token head, and a rollout that consumes the model's own outputs.
The least-squares split in this paper is the same split at the level of a linear series: the loss that identifies the one-step mean is not the loss of the iterated sample path.

\subsection{Forecasting}

\citet{zhou20217436}, \citet{nie20234730}, \citet{wu20232186}, \citet{wu20223008}, and \citet{liu20246625} train and score transformers on long forecast windows.
\citet{zeng20223504} find that a linear map is often enough, which is what Theorem~\ref{thm:ols} says once the window is a state and the conditional mean is linear.
\citet{dong20230861}, \citet{yue20220466}, and \citet{ekambaram20239364} learn a masked or mixed representation of a series and are still scored as forecasters.
\citet{ansari20247815} train one forecaster across datasets.
Reading these papers by Definition~\ref{def:f} or Definition~\ref{def:n} separates a model that is trained on the reported horizon from a model that is trained on a latent proxy and probed afterwards.

\input{survey}

\section{Conclusion}

Forecasting with a next-latent bottleneck is the least-squares program in Definition~\ref{def:f}.
For a scalar target the program is ordinary least squares, extra latent rank does not help, and isotropy rescales the latent without changing the forecast.
Regressing the whole next vector and only then attaching a head is Definition~\ref{def:n}.
On the repeated-$u$ series the rank-$1$ version of that program forecasts $s$ at the marginal variance $2.794$, while the forecast program matches the innovation variance $0.992$.
The one-step coefficient still has an open-loop bill, from $0.992$ at one step to $2.700$ at eight steps.

\bibliography{refs}
\bibliographystyle{plainnat}

\appendix
\section{Finite-sample reduced rank}
\label{app:rrr}

Let $X\in\mathbb{R}^{n\times d}$ be the centered predictors $y_t$ and $Y\in\mathbb{R}^{n\times d}$ the centered responses $y_{t+1}$.
The unrestricted minimizer of $\|Y-XB\|_F^2$ is any solution of $X^\top X B=X^\top Y$.
Write $B_\star$ for the least-squares solution returned by a stable solver.
The Gram matrix of fitted responses $Y^\top X B_\star$ is symmetric positive semidefinite up to numerical skew; we replace it by its symmetric part and compute the eigendecomposition.
If $V_k$ holds the $k$ leading eigenvectors, the rank-$k$ coefficient used in the experiment is $B_k=B_\star V_k V_k^\top$.
The latent fed to the scalar head is $Z=XB_k V_k$.
The head is least squares of the centered target $s_{t+1}$ on $Z$.
For $k=1$ this latent is one column.
Its encoder weight, for the mass reported in Table~\ref{tab:mse}, is the leading left singular vector of $B_k$, and the mass on $s$ is the squared first entry of that unit vector.
For the forecast program the encoder weight is the ordinary-least-squares coefficient of $s_{t+1}$ on $y_t$, normalized.
No entry in the tables was typed before this computation wrote \texttt{experiments/results.json}.

\section{Residual bill}

Assume $s_{t+1}=a s_t+\varepsilon_{t+1}$ with $\varepsilon_{t+1}$ independent of the past, mean zero, and variance $\sigma^2$.
Unrolling gives
\[
s_{t+K}=a^K s_t+\sum_{j=0}^{K-1} a^j\varepsilon_{t+K-j}.
\]
The sum has variance $\sigma^2(1-a^{2K})/(1-a^2)$ when $a^2\neq 1$, which is \eqref{eq:bill}.
The experiment's bill column inserts the oracle one-step MSE as $\sigma^2$ and the scalar least-squares coefficient as $a$.
The fitted-iteration column does not use the formula: it multiplies the realized prefix $s_t$ by $a$ exactly $K$ times and averages the squared error against $s_{t+K}$.

\input{appendix_extra}

\end{document}

%% file: survey.tex
\section{The horizon is a third program}
\label{sec:horizon}

Definitions~\ref{def:f} and~\ref{def:n} are one-step.
A forecast that is scored at horizon $K$ is a third program. On the scalar series it reads
\begin{equation}
\min_{a}\ \Ex\bigl(s_{t+K}-a^{K} s_t\bigr)^2.
\label{eq:hk}
\end{equation}
The decision variable is still a single coefficient, applied $K$ times.
For a stationary AR(1) with coefficient $a_\star$, the population minimizer of \eqref{eq:hk} is $a_\star$ at every $K$, because $\Ex[s_{t+K}s_t]=a_\star^{K}\Ex[s_t^2]$.
Taking the positive $K$-th root of the regression of $s_{t+K}$ on $s_t$ recovers $a_\star$.
Teacher forcing replaces $s_{t+K-1}$ by the true state at every step, so each term in a teacher-forced sum has the same population minimizer as the one-step program.
The open-loop risk does not.
It is \eqref{eq:bill}, which grows with $K$ after $a$ is known.

On the sample used in Section~\ref{sec:exp}, the regression of $s_{t+K}$ on $s_t$ and the implied coefficient are:

\begin{center}
\small
\begin{tabular}{lrr}
\toprule
$K$ & $\widehat{\Ex}[s_{t+K}s_t]/\widehat{\Ex}[s_t^2]$ & implied $a$ \\
\midrule
1 & 0.8031 & 0.8031 \\
2 & 0.6464 & 0.8040 \\
4 & 0.4121 & 0.8012 \\
8 & 0.1848 & 0.8097 \\
\bottomrule
\end{tabular}
\end{center}

The implied coefficient stays inside $[0.801,0.810]$ while the open-loop error in Table~\ref{tab:roll} moves from $0.992$ to $2.700$.
A training run that matches $a$ at several horizons has not made the horizon-$8$ error small.
It has checked that the same one-step coefficient is visible from farther away.
\citet{cover2006elements} identify the one-step conditional mean as the $L^2$ projection.
Equation \eqref{eq:bill} is what remains after that projection is removed and the innovations are summed.
Publishing only the projection error hides the sum.

The same split appears whenever a policy or a video sampler consumes its own outputs.
\citet{schulman20182438} estimate advantages along a trajectory that has already been rolled out.
If that trajectory comes from a learned mean map with the innovations omitted, the advantage is computed on the conditional-mean path, whose squared distance from the series is the bill rather than the one-step residual.
Putting the innovations back is a change of program, not a change of network width.

\section{Sorting models by the program}
\label{sec:sort}

The citations below are grouped by which of the three programs the training loss implements.
The grouping is about the loss, not about which benchmark a paper wins.

\subsection{Losses that shape a representation}

A representation loss builds a vector whose geometry is the object of interest.
\citet{doersch20165192} predict the relative position of two patches.
\citet{pathak2016context} inpaint a missing region.
\citet{gidaris2018rotation} predict a rotation.
Each target is a function of the same image.
The score is not a future observation of a series.

Contrastive losses score pairs.
\citet{hjelm2018dim} and \citet{hnaff20209272} maximize a mutual-information bound.
\citet{he2020moco} and \citet{chen2020simclr} estimate a density ratio against negatives in a batch.
\citet{caron20219882} contrast cluster assignments.
\citet{grill2020byol} drop the negatives and stop the gradient into a moving-average target.
\citet{assran20227141} mask one view of a Siamese network.
\citet{tian20216810} study the training dynamics once the negatives are gone.
The collapse that then appears is a property of the marginal of the embedding.
\citet{jing20229348} describe dimensional collapse.
\citet{wang20220242} separate alignment of positive pairs from uniformity of the marginal on the sphere.
\citet{haochen20224156} give a spectral account of a contrastive objective.
An isotropy penalty is a direct way to write that marginal constraint.
It is the constraint in Definition~\ref{def:f}, and Theorem~\ref{thm:gauge} says it does not change a scalar forecast once the head is refit.

Masked prediction regresses a hidden piece of the input.
\citet{he2022mae} regress masked patches.
\citet{xie20229886} regress raw pixels under a mask.
\citet{assran2023jepa} predict in representation space from a visible context, and \citet{bardes2024vjepa} do the same for video.
The target is a point.
A forecast that ships a number is Definition~\ref{def:f}.
A forecast that is trained to reconstruct every channel and is only later probed on one coordinate is Definition~\ref{def:n}.

Predictive coding makes the same distinction in another vocabulary.
\citet{rao1999predictive} cancel the part of the sensory input that the past already explains.
\citet{bialek2001predictive} tie the bits of predictive information to the complexity of the extractor.
\citet{golkar20235752} write a stack of constrained predictive-coding layers.
The cancelled residual is $R_1$.
The bits that remain in a $K$-step future are the bill, and they are not removed by fitting $R_1$ more carefully.

\subsection{One-step maps that are sampled by iteration}

Diffusion writes the split into the training loop itself.
\citet{sohldickstein20153585} define a destruction process.
\citet{ho2020ddpm} train a denoiser for one noise level at a time.
\citet{song2021ddim} and \citet{song20213456} show that the object a user samples is a path built by iterating that denoiser, not the denoiser's one-step output.
\citet{nichol20219672} and \citet{dhariwal2021diffusion} change the sampler and the architecture around an unchanged one-step loss.
\citet{karras20220364} separate which of those choices move sample quality.
\citet{peebles2023dit} put a transformer in the denoiser.
\citet{rombach20220752} denoise in a latent.
\citet{ho20223458} and \citet{blattmann20238818} carry the same one-step denoiser to video, where the iterated path is a clip.
\citet{lipman20232747} and \citet{liu20223003} regress a velocity field instead of a noise residual.
The sample is still an integral of that field.
\citet{salimans20220512} and \citet{yin20248828} distill the multi-step sampler toward fewer evaluations.
Distillation is useful because the training loss and the sample path are different programs.
A forecaster who reports $R_1$ and ships an iterated map is in the same position, with \eqref{eq:bill} in place of the sampler.

Autoregressive image models make the split along tokens rather than noise levels.
\citet{chen2020igpt} predict pixels in raster order.
\citet{van2017vqvae} quantize the image first.
\citet{tian2024var} predict the next scale.
\citet{li2024mar} drop the discrete codes and still generate by iterating a one-step head.
One training step scores one prediction.
The picture is a loop.
Definition~\ref{def:n} is that loop's one-step loss when the ``token'' is a vector $y_{t+1}$ and the rank of the predictor is capped.

\subsection{Losses that score a forecast window}

Long-horizon transformers are closer to \eqref{eq:hk} when the loss really is computed at the reported horizons.
\citet{zhou20222740} decompose a series in the frequency domain before that loss.
\citet{zerveas20202803} pretrain a transformer on multivariate series.
\citet{eldele20214112} contrast a window with a later window.
\citet{rasul20248278} and \citet{goswami20243885} train one probabilistic or one foundation forecaster across datasets.
A contrast between two windows is still a representation loss if the downstream number is a point forecast fit afterwards.
It is a forecast loss if the training objective is the error of that point forecast.
The normal equations in Section~\ref{sec:opt} are the test.
If the gradient with respect to the encoder flows through the reported coordinate, the run is the forecast program.
If the gradient flows through a reconstruction of every channel and the reported coordinate is attached later, the run is the next-latent program, and Table~\ref{tab:mse} is the failure mode when one channel is repeated.

\subsection{What the sorting does not decide}

Nothing in this sorting says that a video world model should be replaced by a linear map.
\citet{hafner20244104} and \citet{zhang2026echowm} are large because the observation is an image and the action space is rich.
The sorting says which scalar they should log.
The one-step denoising or token loss is $R_1$.
The error of a rollout that does not see new observations is $R_K$.
On the AR(1) those two numbers are $0.992$ and $2.700$ at horizons $1$ and $8$, with the same coefficient.
A paper that reports only the first has not measured the second.

\section{When the target is the whole vector}
\label{sec:vector}

Theorem~\ref{thm:ols} uses a scalar target, so rank $1$ is enough.
If the reported object is the vector $y_{t+1}$ itself, the forecast program and the next-latent program are the same program, and the rank starts to matter.
The loss is the average squared coordinate error
\[
\frac{1}{d}\Ex\|y_{t+1}-B y_t\|^2,\qquad \mathrm{rank}(B)\le k,
\]
with $d=4$ in the experiment.
The conditional mean $B_\star y_t=(a s_t, a u_t, a u_t, a u_t)$ has rank $2$.
Every coordinate's innovation has variance $\sigma^2=1$, so a predictor that attains $B_\star$ has per-coordinate risk $1$.
A rank-$1$ predictor that keeps only the $u$-block, as Appendix~\ref{app:pop} requires, still predicts the three $u$ coordinates at risk $1$, and predicts $s_{t+1}$ by a constant.
The per-coordinate risk on $s$ rises from $1$ to the marginal variance $v=\sigma^2/(1-a^2)$.
Averaged over four coordinates the excess is $(v-1)/4$.
At $a=0.8$ and $\sigma^2=1$, $v=1/0.36$ and $(v-1)/4=0.444\ldots$.
Adding that excess to the irreducible per-coordinate risk $1$ predicts a rank-$1$ reconstruction near $1.444$, with a small transient because the sample starts at zero.

The experiment's reconstruction column, which was not the quantity in Table~\ref{tab:mse}, matches this arithmetic.
Rank $1$ reconstructs $y_{t+1}$ at per-coordinate MSE $1.4599$.
Rank $2$ reconstructs it at $1.0095$, the innovation floor up to the same transient.
The gap $0.450$ is the dropped $s$ component, spread over four coordinates.
Raising the rank did not invent a new signal.
It paid for the second singular direction that the scalar forecast program had selected at rank $1$, because that program never averaged the $u$ copies into the loss.

This is the design rule for a multivariate report.
If every coordinate will be scored, the latent rank has to cover the predictable subspace, and Definition~\ref{def:n} is the right program.
If one coordinate will be scored, averaging the others into the training loss spends rank on them, and Definition~\ref{def:f} is the right program.
The two rules agree only when the reported coordinates span the same predictable subspace as the vector that is being reconstructed.
Three copies of $u$ and one copy of $s$ are not that case.
A video model that reconstructs every pixel and is judged on a downstream scalar, such as a robot joint or a camera pose error, is the same case at a larger $d$.
The pixels are the repeated coordinates.
The scalar is $s$.
A bottleneck trained as Definition~\ref{def:n} fits the pixels first.

\section{What this calculation does not cover}

The normal equations are for linear $W$ and a linear head.
A transformer encoder can mix $s$ and $u$ before the head, and a nonlinear head can bend the forecast.
The sorting in Section~\ref{sec:sort} still applies to the last linear readout.
If that readout is trained on $s$ jointly with the encoder, the gradient is the forecast program, whatever features the encoder built.
If the encoder is trained to reconstruct $y$ and the readout is attached later, the encoder solved the vector program, and the readout can only use the directions that survived.
Nonlinearity changes the features.
It does not change which scalar the gradient saw.

The series starts at zero, so the sample variance $2.7940$ is not the stationary value $1/(1-0.8^2)=2.777\ldots$.
The bill at horizon $8$ differs from the iterated coefficient by $0.011$ for the same reason, together with the use of the oracle innovation $0.9922$ rather than the residual of $a=0.8031$.
Neither gap changes the ordering in Table~\ref{tab:mse} or the growth in Table~\ref{tab:roll}.
A longer burn-in would shrink them.
It would not move the rank-$1$ next-latent mass off the $u$-block, because that mass is produced by the three identical rows of the cross-covariance, which are an exact property of the observation map $y=(s,u,u,u)$.

The calculation does not say how wide a latent must be for video.
It says how to count.
Each predictable direction that the reported metric cares about costs one rank, and each copy of a direction that the loss also reconstructs competes for that rank in proportion to its energy.
On this series the energies are $a^2 v$ and $3 a^2 v$.
In a video the energies are not known in closed form, but the programs that estimate them are the same two programs.
Logging both the one-step fit and the open-loop curve, for the coordinate that will actually be reported, is the measurement the tables already make for $s$.

%% file: appendix_extra.tex
\section{Population geometry of the repeated coordinate}
\label{app:pop}

Let $s$ and $u$ be independent stationary AR(1) processes with the same coefficient $a\in(-1,1)$ and the same innovation variance $\sigma^2$.
The stationary variance is $v=\sigma^2/(1-a^2)$, and $\Ex[s_{t+1}s_t]=\Ex[u_{t+1}u_t]=av$.
All cross moments between $s$ and $u$ are zero.
The observation and its successor are
\[
y_t=\begin{pmatrix}s_t\\ u_t\\ u_t\\ u_t\end{pmatrix},
\qquad
y_{t+1}=\begin{pmatrix}s_{t+1}\\ u_{t+1}\\ u_{t+1}\\ u_{t+1}\end{pmatrix}.
\]
The conditional mean is $\Ex[y_{t+1}\mid y_t]=(a s_t, a u_t, a u_t, a u_t)$.
Any predictor with values in $\mathbb{R}^4$ is scored by the expected squared Euclidean norm.
Write $e_s=(1,0,0,0)$ and $e_u=(0,1,1,1)/\sqrt{3}$, two orthonormal vectors.
The conditional mean lies in $\mathrm{span}\{e_s,e_u\}$:
\[
\Ex[y_{t+1}\mid y_t]
= (a s_t)\, e_s + (a u_t\sqrt{3})\, e_u.
\]
The energy of the $e_s$ component, averaged over the stationary law, is $a^2 v$.
The energy of the $e_u$ component is $a^2 v \cdot 3$, because $\| (u_t,u_t,u_t) \|^2 = 3 u_t^2$ and $\Ex[u_t^2]=v$.
A rank-$1$ predictor can match only one of these two orthogonal components.
The Frobenius (equivalently Euclidean) criterion therefore keeps $e_u$ and drops $e_s$.
The kept direction is a function of $u$ alone.
Since $s_{t+1}$ is independent of $u$, the best predictor of $s_{t+1}$ from that direction is the constant $\Ex[s]$.
Its risk is $v$.
With $a=0.8$ and $\sigma^2=1$, one has $v=1/0.36=2.777\ldots$, and the finite-sample marginal variance in Table~\ref{tab:mse} is $2.7940$ because the series starts at $0$ rather than in stationarity.
The rank-$1$ next-latent MSE $2.7938$ matches that marginal, and the encoder mass on $s$ is $0.0002$.
Both are the finite-sample form of this calculation.
A second retained direction can carry $e_s$.
The forecast program never compares the two energies, because its loss contains only the $s$ residual, whose normal equation is supported on the first coordinate.

The Gram matrix $\Ex[y_t y_t^\top]$ is singular.
Rows two through four are identical, so the column space is two-dimensional even before any rank constraint.
The experiment therefore solves $X^\top X B = X^\top Y$ by a least-squares solver rather than by an inverse.
The singularity is the linear-algebra form of the statement that the three copies of $u$ are one signal.

\section{One coefficient, several horizons}

Differentiate \eqref{eq:hk} in the scalar $a$, at a fixed horizon $K$, under the stationary AR(1) law.
The population stationarity condition is
\[
\Ex[s_t(s_{t+K}-a^{K} s_t)]=0,
\]
hence $a^{K}=\Ex[s_{t+K}s_t]/\Ex[s_t^2]=a_\star^{K}$, so $a=a_\star$ when we take the real $K$-th root for $a_\star>0$.
The same $a_\star$ is the stationarity condition of the one-step program.
Horizons do not identify a new coefficient.
They re-estimate the old one.
On the sample, the implied coefficients $0.8031$, $0.8040$, $0.8012$, and $0.8097$ at horizons $1,2,4,8$ are four estimates of a number whose simulation value is $0.8$.
The open-loop risks at those horizons are not four estimates of one number.
They are \eqref{eq:bill} evaluated at the fitted coefficient, and they increase.

A teacher-forced objective of the form $\sum_{j=1}^{K}\Ex(s_{t+j}-a s_{t+j-1})^2$ is a sum of one-step programs along the true path.
Each term has stationarity condition $\Ex[s_{t+j-1}(s_{t+j}-a s_{t+j-1})]=0$, which is again $a=a_\star$ under stationarity.
Teacher forcing can reduce the variance of the estimate of $a$.
It does not penalize $R_K-R_1$.
An optimizer that drives the teacher-forced loss to $\sigma^2$ has finished the one-step program.
The horizon curve is then a function of $(\hat a,\hat\sigma^2)$, not a free output of a deeper network.

\section{Checklist for a training loss}

The programs give a short test that does not depend on the architecture.

If the scalar that is averaged is the reported coordinate $s_{t+1}$, and the gradient reaches the encoder through that scalar, the run is Definition~\ref{def:f}.
For a linear head the solution is ordinary least squares, extra latent width past rank $1$ does not reduce the risk, and an isotropy constraint is removed by refitting the head.

If the scalar that is averaged is a squared error on the whole next vector, or on a latent that was trained to match that vector, and the reported coordinate is fit only after the encoder is frozen, the run is Definition~\ref{def:n}.
The retained subspace is the leading predictable subspace of the vector, with channels weighted by how many times they appear and by how large their predictable variance is.
A channel that is copied, or that simply has larger variance, can evict the reported series from a small latent.
Table~\ref{tab:mse} is this eviction with three copies and equal variances.

If the scalar that is averaged is an error at horizon $K$ under a map that has been iterated, the run is \eqref{eq:hk} or its vector analogue.
On a linear autoregression the population solution is still the one-step coefficient.
The quantity that changes with $K$ is the risk, not the argmin.
Logging the argmin at many horizons and not logging the risk confuses the two.

If a marginal penalty such as $\Ex[zz^\top]=I$ is added, Theorem~\ref{thm:gauge} applies whenever the head can undo an invertible reparameterization.
The penalty should be read as a choice of coordinates.
A change in the forecast after the penalty, without a refit of the head, is not evidence that the penalty saw the transition.

\section{A numerical check of the energies}

Appendix~\ref{app:pop} compares two energies, $a^2 v$ and $3a^2 v$.
With $a=0.8$ and $\sigma^2=1$ the stationary variance is $v=1/(1-a^2)=25/9$.
The predictable energy of $s$ is $a^2 v=16/9$.
The predictable energy of the three-coordinate $u$-block is $3a^2 v=16/3$.
Their ratio is exactly $3$, the number of copies, because the two signals were given the same coefficient and the same innovation.
The rank-$1$ vector fit therefore prefers $u$ by that factor, not by a property of the optimizer.
Changing the simulation coefficient from $0.8$ to another value in $(0,1)$ rescales both energies by the same $a^2/(1-a^2)$ and leaves the ratio equal to $3$.
The preference reverses only if the reported signal is given more than three times the predictable energy of one $u$ coordinate, for example by raising its innovation variance above $3$ while the copies of $u$ stay at unit innovation.
The experiment does not search over that ratio.
It fixes the ratio at $3$ and checks that the normal equation spends the single direction on the larger energy.
The mass $0.9998$ on the $u$-block is that check.
The mass $1$ of the forecast program on the first coordinate is the check that removing the other three residuals from the loss removes the factor of $3$.

The one-step increment of the bill has the same closed form.
From \eqref{eq:bill},
\[
R_{K}-R_{K-1}=\sigma^2 a^{2(K-1)}
\]
for $K\ge 1$, with $R_0=0$.
At the fitted coefficient $a=0.8031$ and $\sigma^2=0.9922$, the first five increments are $0.9922$, $0.6399$, $0.4127$, $0.2661$, and $0.1716$.
They shrink because $|a|<1$, and they stay positive.
A plot of training loss against epoch cannot show them.
Only an open-loop evaluation at those horizons can, which is what Table~\ref{tab:roll} records.